\documentclass[letterpaper, 10pt, conference]{ieeeconf}  
\IEEEoverridecommandlockouts 
\usepackage[utf8]{inputenc}
\usepackage{amsmath}
\usepackage{amsfonts} 
\usepackage{amssymb} 
\usepackage{mathtools} 
\usepackage{colonequals} 
\usepackage{textgreek} 
\usepackage{graphicx} 
\usepackage{siunitx} 
\usepackage{subcaption} 
\usepackage{xspace}
\usepackage{xcolor} 
\usepackage{adjustbox}
\usepackage{acronym}
\usepackage{float} 
\usepackage{algorithm}
\usepackage[noend]{algorithmic}
\usepackage[font=small]{caption}
\usepackage{soul}
\usepackage{comment}
\usepackage{cite}
\usepackage{oldstyle} 
\usepackage{feyn} 
\usepackage{pbalance}

\newcommand{\gobble}[1]{}
\newcommand{\gobblexor}[2]{#2} 

\usepackage{array} 
\newcolumntype{R}[2]{%
    >{\adjustbox{angle=#1,lap=\width-(#2)}\bgroup}%
    l%
    <{\egroup}%
}

\newtheorem{corollaryenv}{\bf Corollary}
\newtheorem{lemmaenv}{\bf Lemma}
\newtheorem{exampleenv}{\bf Example}
\newtheorem{definitionenv}{\bf Definition}
\newtheorem{remarkenv}{\bf Remark}
\def\markatright#1{\leavevmode\unskip\nobreak\quad\hspace*{\fill}{#1}}
\newenvironment{proofsketch}{\par\emph{\quad Proof sketch:}}{\markatright{$\Box$}\par}

\usepackage{xparse}
\ExplSyntaxOn
    \NewDocumentEnvironment{definition} { o }
     {
      \IfNoValueTF{#1} {\vspace{0.05in}\begin{definitionenv}\em}{\vspace{0.05in}\begin{definitionenv}[#1]\em}
     }
     {\end{definitionenv}\vspace{0.05in}}
\ExplSyntaxOff

\newenvironment{lemma}{\vspace{0.05in}\begin{lemmaenv}\em}{\end{lemmaenv}\vspace{0.05in}}

\newenvironment{corollary}{\vspace{0.05in}\begin{corollaryenv}\em}{\end{corollaryenv}\vspace{0.05in}}
\newenvironment{example}{\vspace{0.05in}\begin{exampleenv}}{\end{exampleenv}\vspace{0.05in}}

\newenvironment{remark}{\vspace{0.05in}\begin{remarkenv}}{\end{remarkenv}\vspace{0.05in}}
\newcommand{\powSet}[1]{\raisebox{.15\baselineskip}{\large\ensuremath{\wp}}({#1})}
\renewcommand{\emptyset}{\varnothing}

\definecolor{propcol}{rgb}{0.05,0.45,0.05}
\definecolor{autocol}{rgb}{0.45,0.05,0.05}
\definecolor{goalcol}{rgb}{1.0,0.6,0.0}

\newcommand{\goal}{\textcolor{goalcol}{\footnotesize\bf\textsf{G}}\xspace}

\def\Nat{{\mathbb{N}}}
\def\Re{{\mathbb{R}}}

\def\ifs{{\mathcal{I}}}
\def\H{{\mathcal{H}}}

\def\X{{\mathcal{X}}}

\def\is{{\iota}}
\def\iS{\scalebox{1.3}{\textiota}}

\providecommand{\LTL}{\textsc{ltl}\xspace}
\providecommand{\ndet}{{\rm ndet}}
\providecommand{\UOmega}{\ensuremath{\overline{\Omega}}}
\providecommand{\bnfeqq}{\ensuremath{\;\coloncolonequals\;}}
\providecommand{\bnfvert}{\ensuremath{\,\;\vert\;\,}}
\providecommand{\prop}{\ensuremath{\textsc{prop}}\xspace}
\providecommand{\lfalse}{\ensuremath{\scalebox{0.90}{$\bot$}}}
\providecommand{\ltrue}{\ensuremath{\scalebox{0.90}{$\top$}}}
\providecommand{\lstart}{\ensuremath{\scalebox{0.9}{\textbf{start}}}\xspace}
\providecommand{\limplies}{\ensuremath{\Rightarrow}}
\providecommand{\liff}{\ensuremath{\Leftrightarrow}}
\providecommand{\lnext}{\ensuremath{\raisebox{0.75pt}{\scalebox{0.82}{$\bigcirc$\thinspace}}}}
\providecommand{\luntil}{\ensuremath{\thinspace\mathsf{U}}}
\providecommand{\leventu}{\ensuremath{\scalebox{0.90}{$\diamondsuit$\thinspace}}}
\providecommand{\lalways}{\ensuremath{\scalebox{0.90}{$\square$\thinspace}}}
\providecommand{\lunless}{\ensuremath{\thinspace\mathsf{W}}}
\providecommand{\prp}[1]{\textcolor{propcol}{\mathrm{\bf #1}}}
\providecommand{\autom}[1]{\ensuremath{{\color{autocol}\mathcal{#1}}}}
\providecommand{\initt}{\text{initial time}}

\providecommand{\Light}{\raisebox{-.4pt}{\includegraphics[scale=0.4]{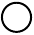}}\xspace}
 \providecommand{\Dark}{\raisebox{-.4pt}{\includegraphics[scale=0.4]{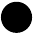}}\xspace}
\providecommand{\Indet}{\raisebox{-.4pt}{\includegraphics[scale=0.4]{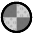}}\xspace}

\title{Limits of specifiability for sensor-based robotic planning tasks}

\author{Ba\c{s}ak Sak\c{c}ak \and Dylan A. Shell \and Jason M. O'Kane\thanks{
The first author is with the Center for Ubiquitous Computing, University of Oulu, Finland. 
The last two are with the
Department of Computer Science and Engineering, Texas A\&M University, College
Station, TX, USA.  {\tt basak.sakcak@oulu.fi} \& {\tt\{dshell, jokane\}@tamu.edu}. The first author was supported by Academy of Finland (project CHiMP 342556).}}

\begin{document}

\maketitle

\begin{abstract}
There is now a large body of techniques, many based on formal methods, for describing and realizing complex robotics tasks, including those involving a variety of rich goals and time-extended behavior.  This paper explores the limits of what sorts of tasks are specifiable, examining how the precise grounding of specifications\,---that is, whether the specification is given in terms of the robot's states, its actions and observations, its knowledge, or some other information---\,is crucial to whether a given task can be specified. While prior work included some description of particular choices for this grounding, our contribution treats this aspect as a first-class citizen: we introduce notation to deal with a large class of problems, and examine how the grounding affects what tasks can be posed. The results demonstrate that certain classes of tasks are specifiable under different combinations of groundings.
\end{abstract}

\section{Introduction}
\vspace*{-3pt}
\gobblexor{A great deal of work in robotics involves determining \emph{how} to get robots
to do things\,---choosing the steps to be taken and enacting them, either in
sequence or as function of a variety of conditions.}{Much work in robotics involves determining \emph{how} to get robots
to do things\,---choosing the steps to be taken, enacting them in
sequence or under suitable of conditions.}
In concert, there is also distinct and growing body of work dedicated to
expressing \emph{what} we wish the robot to do, viz.\ describing the objectives
of robot system.  Common techniques include describing goal regions in state
space, providing reward functions associated to states (or states and actions),
and rich languages and logics that enable formal specifications of sequences.
In the present paper, we examine the fundamental limits of these sorts of task
descriptions.
Such questions are essential but often overlooked: When one is unable to
characterize the demands to be placed upon the robot, one should not expect the desired behavior to be elicited.

\gobblexor{Figure~\ref{fig:coastal} depicts an example in which a robot moves
through an environment, aided by a sensor that can detect landmarks in known
locations.}{Figure~\ref{fig:coastal} depicts an example of a robot moving
through an environment, aided by a sensor that can detect landmarks in known
locations.}
Such a robot might be tasked with localizing itself, though there is some
subtlety in what exactly constitutes success: Must the robot come to know its
own state and maintain that knowledge through the rest of its execution ---in
the language of temporal logic, `eventually always' localized--- or it is
sufficient for the robot to know its own state periodically ---`always
eventually' localized?
Of the \gobble{robot's }two trajectories shown\gobble{in the figure}, both satisfy the latter
requirement, but only Trajectory~{\textsf A} satisfies the former.
This simple example shows both the need for immaculate precision in task
specification, lest the elicited behaviour mismatch the intent, and the
potential for expressive representations like temporal logic to capture the
relevant subtlety.

To that end, this paper introduces a formal structure in which questions about
specifiability of tasks for robots may be understood in a precise way.  This
structure begins from the complete, infinite traces of a robot's interaction
with its environment across time, including the evolution of the robot's
states, along with the robot's actions 
(which influence this state evolution)
and observations (which may partially reveal the state evolution to the robot).
In this abstract sense, a task is simply a collection of such traces that
correspond to completion of the objective.

\begin{figure}
    \centering
    \vspace*{4pt}
    \includegraphics[trim={0cm 0cm 4cm 0},clip,scale=0.75]{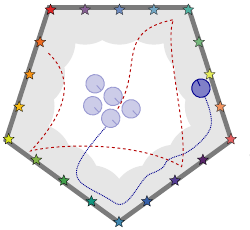}
    \includegraphics[trim={4.0cm -0.4cm 0cm 0},clip,scale=0.65]{figures/coastal-navigation-world.pdf}
    \vspace*{-4pt}
    \caption{%
        A representative problem domain.
        Starting from an initial position with some uncertainty, a robot with
        imperfect motion dynamics navigates around an environment whose
        boundary contains landmarks that may be detected and distinguished from
        one another. 
        The essence of the setting is the interplay of the robot's states,
        actions, and observations.
        This paper addresses the ability of various specification approaches to
        describe tasks such as localization, coverage, or navigation, in this
        sort of setting.
    }
    \label{fig:coastal}
    \vspace*{-1.8em}
\end{figure}

The contribution of this paper is to provide a direct connection between this
very broad concept of robotic tasks and more commonly-utilized specification
languages such as linear temporal logic (\LTL).  We show that the precise set
of atomic propositions used to ground such logical formulae impacts the tasks
that can be specified, in direct and sometimes counter-intuitive ways.
We consider tasks specified in terms of states, in terms of action/observation
histories (generalizing LaValle's history information states \cite{Lav06} to infinite
traces), and so-called information states which refer to sets of possible
states.
The results include a characterization of the existence or non-existence of
tasks specifiable within various subsets of these three example
specification types.

Following a concise review of related work (\S\ref{sec:related}), the
remainder of this paper introduces some preliminary definitions
(\S\ref{sec:prelim}), defines tasks and their posability
(\S\ref{sec:task}) and specifications (\S\ref{sec:spec}).  Upon
that foundation, we prove several results about the limits of specifiability
(\S\ref{sec:table}) before offering concluding remarks
(\S\ref{sec:conc}).

\vspace*{-3pt}
\section{Related work}\label{sec:related}
\vspace*{-3pt}

\gobblexor{
Though recent years have witnessed a surge in interest specifications of tasks
for robots using natural
language~\cite{liu2022langltl,wang2024llm,hu2024deploying,arkin2017contextual,liu2023grounding},
there has been sustained interest in various forms of specifying desired robot behaviors.
Careful study has been devoted to tasks specified as collections of
partially-ordered rules that specify behaviours to which a robot should
adhere~\cite{censi2019liability}, formal contracts that describe desired
interactions between subsystems~\cite{ghasemi2024compositional}, specifications
on the narrative structure of events observed by the robot~\cite{RahSheOKa22},
specifications that blend partially-observable Markov decision process and
temporal logic~\cite{liu2021leveraging}, generalized domain specifications that
apply to large classes of specific tasks with certain forms of shared
structure~\cite{curtis2022discovering}, and even specifications \gobble{of robot
movement }inspired by the study of human dance~\cite{6016594,5991363}.
}
{
Despite a recent surge in interest to describe tasks
for robots using primarily natural
language~\cite{liu2022langltl,wang2024llm,hu2024deploying,arkin2017contextual,liu2023grounding},
there has been sustained interest in a surprisingly wide variety of forms of specification for many years.
Careful study has been devoted to tasks specified as collections of
partially-ordered rules that require adherence~\cite{censi2019liability}, contracts describing desired
subsystem interactions~\cite{ghasemi2024compositional}, specifications
on the narrative structure of events observed by the robot~\cite{RahSheOKa22},
blends of POMDPs and temporal logic~\cite{liu2021leveraging}, and domain
specification generalization across classes of
tasks~\cite{curtis2022discovering}, and even specifications borrowed from the
study of human dance~\cite{6016594,5991363}.
}

From this wide array of options, one common choice for describing the desired
behaviour of a robot\,---borrowed from the formal methods community---\,is \LTL,
a logic particularly suitable for specifying robot tasks with temporally
extended goals.
A large body of robotics research exists that relies on \LTL specifications including within the context of mobile robot navigation~\cite{FaiGirKrePap09},
multi-agent systems~\cite{ulusoy2013optimality,kantaros2022perception}, and
even to describe rules for changing the environment that the robot interacts
with~\cite{bobadilla2012controlling}.
Other temporal logics have also been of interest, including signal temporal
logic because of its suitability for continuous
systems~\cite{liu2023learning,sun2022multi}. 

There have been at least two major efforts to address questions about whether
tasks of interest can be expressed within certain specification schemes.
One approach is based on the study of various fragments of \LTL and their
ability to express certain classes of tasks~\cite{10571628,6942755}.
Another analyzed the expressivity of reward functions for prescribing desired
behavior, including the \emph{reward
hypothesis}~\cite{abel2021expressivity,bowling2023settling}. 
This paper explores a different avenue; it focuses on the grounding of
propositions that appear in a general class of specifications and connects them back to the
basic descriptive model of the robot's interactions with the environment. Consequently, we
discuss cases in which certain groundings prove insufficient to express a
task of interest.
In considering this grounding, our work shares a commonality with other work on 
varying levels of spatial abstraction in the propositions within a temporal
logic specification~\cite{oh2022hierarchical}. 
Whereas other work's focus is on the synthesis of task specifications, our results establish conditions under which varying groundings are sufficient for expressing a particular task.
Likewise, the results are orthogonal to other work that considers the case in
which a duly specified task admits no correct plan for the robot to complete
it~\cite{raman2011analyzing}.

\section{Preliminaries}\label{sec:prelim}
This section lays a foundation for our results by introducing a formal notion of sensor-based robot tasks.

First, some notational niceties.
For a given set $A$, we write $A^{\Nat}$ to denote the set of all (one-sided)
infinite sequences of elements of set $A$.  
%
%
For a function $g : A \rightarrow B$, $g^{-1}(b) \subseteq A$ is the preimage
of $b \in B$ under $g$. Furthermore, for $C \subseteq A$, we will use $g(C)$ to denote the image of $C$
under $g$. The powerset of $A$ is denoted by $\powSet{A}$.
When referring to a sequence of tuples we will drop the parentheses whenever it
is clear from the context. For example, a sequence of pairs $( (a_1, b_1),
(a_2, b_2), \dots) \in (A \times B)^\Nat$ will be equivalently denoted $(a_1,
b_1, a_2, b_2, \dots)\gobble{ \in (A \times B)^\Nat}$.

\subsection{Robot transition systems and complete traces} 
Our treatment's cornerstone is the robot transition system:

\begin{definition}[Robot transition system]
    A \emph{robot transition system} is a tuple $R = (X, U, f, h, Y, X_0)$
    consisting~of
    \begin{itemize}
        \item Sets of states $X$, actions $U$, and observations $Y$. 
        \item A transition function $f: X \times U \to
            \powSet{X}\setminus\{\emptyset\}$, where $f(x,u)$ is the set of
            possible states that can be reached when action $u$ is taken by the
            robot in state $x$.
        \item An observation function $h: X \times U \times X \to \powSet{Y}\setminus\{\emptyset\}$,
            where $y \in h(x,u, x')$ is an observation the robot receives after
            taking action $u$ in state $x$ and arriving in state $x'$.
        \item A set of initial states $X_0 \subseteq X$.
    \end{itemize}
\end{definition}

The idea is that our robot interacts with its environment, executing actions and receiving observations.
\gobblexor{Notice that this}{This} model can express both nondeterminism in its state
transitions\,---from state $x$ executing action $u$, the next state $x'$ can be
any of the states in $f(x, u)$---\,and nondeterminism in the sensor information
available to the robot\,---when transitioning from $x$ to $x'$ by executing
action $u$, the observation $y$ received by our robot may be any of the
observations in $h(x, u, x')$.

\begin{example}[Robot transition system]\label{eg:RTS_coastal}
    Recall the example in Figure~\ref{fig:coastal}.
    To model this scenario as a robot transition system, let $X = E \times
    S^1$, in which $E \subseteq \Re^2$ is the environment and $S^1$ represents
    orientation.
    For its actions, the robot can rotate to change its heading and then
    translate forward in each step, so its action space is $U = S^1 \times
    (0,d]$ for some maximum translation $d \in \Re_+$. The state transition
    function $f$ returns the results of this requested motion, including some
    inexactitude in the movements.
    This robot's sensing may be modeled with the observation space
        $Y = \left( \{
            {\includegraphics[trim={0.01cm 0cm 7.8cm 0},clip,scale=1.0]{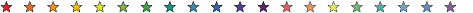}},
            {\includegraphics[trim={1.2cm 0cm 6.6cm 0},clip,scale=1.0]{figures/landmark-stars.pdf}},
            \ldots,
            {\includegraphics[trim={3.17cm 0cm 4.63cm 0},clip,scale=1.0]{figures/landmark-stars.pdf}}
        \} \times \Re^2 \times S^1 \right) \cup \{ \bot \}$,
    in which $h$ reports the identity and relative pose of the nearest
    landmark, if any are within a given maximum sensing range, or $\bot$ if no
    landmarks are in range.
\end{example}

To ensure that tasks \gobblexor{involving interactions}{evolving} over indefinite periods of time
(e.g., maintenance tasks) can be modeled\gobblexor{ correctly,}{,} we are interested in
infinite traces that describe the robot's execution. 

\begin{definition}[Universe and complete trace] A \emph{complete trace} (or, just, a trace) for a robot transition system $R=(X,U,f, h, Y, X_0)$ is an infinite sequence of state, action, and observation triples of the form $\omega = (x_1, u_1, y_1, x_2, u_2, y_2, \dots) \in (X \times U \times Y)^\Nat$ that satisfy $x_1 \in X_0$ and respect $f$ and $h$ such that for all $i= 1, 2, \dots$, $x_{i+1} \in f(x_i, u_i)$ and $y_{i} \in h(x_i,u_i, x_{i+1})$.
    %
    We denote by $\Omega_R$ the set of all complete traces for $R$, and $\Omega_R \subseteq (X\times U \times Y)^\Nat \eqqcolon \UOmega_R$, the \emph{universe of traces}.
\label{def:traces}
\end{definition}
\gobblexor{Such traces are `complete' in the sense that they include the full detail of how the robot's execution unfolds, including states, actions, and observations.}{Such traces are `complete' in their full detail of how the \gobble{robot's }execution unfolds: all states, actions, and observations.}

\vspace*{-1.0ex}
\subsection{Condensers}\label{sec:condenser}
\vspace*{-0.8ex}

A central idea in this paper is the importance of reasoning about what can be accomplished with only a limited view of the complete trace $\omega$.  \gobblexor{The next definition makes this concept precise.}{Next, we make this precise.}

\begin{definition}[Condenser]\label{def:condenser}
    For a given robot transition system $R$ and an arbitrary set $A$, a \emph{condenser} for $R$ to $A$ is a function with domain $\Omega_R$ and codomain $A$.
 
\end{definition}

\gobble{
When a function akin to a condenser is defined over the larger domain $\UOmega_R$ instead, it will a termed a \emph{munificent condenser}. Its restriction to $\Omega_R$ yields a condenser.}

The next two examples illustrate the concept of a condenser by introducing important special cases.

\begin{example}[Action/observation-trace]
Let $\H = (U \times Y)^\Nat$ denote the \emph{action/observation-trace space} of a given robot transition system $R$.
The \emph{action/observation-trace condenser} $c_\H: \Omega_R \to \H$ discards the states from each complete trace, leaving only the actions and observations, so that
\vspace*{-1.75ex}
        $$ (x_1,u_1,y_1,x_2, u_2, y_2, \dots) \mapsto (u_1, y_1, u_2, y_2, \dots).%
\vspace*{-1.5ex}$$
This condenser captures the idea that a robot generally will not have direct information about its states, but will have access to history of actions executed and observations received.  
The concept is closely akin to LaValle's history information space~\cite{Lav06}, but expressed over infinite traces.
\end{example}

\begin{example}[State-trace]
    Let $\X = X^\Nat$ denote the 
    \emph{state-trace space} 
    of a given robot
    transition system $R$. The \emph{state-trace condenser} $c_\X:
    \Omega_R \rightarrow \X$ maps a complete trace to the respective
    state-trace by discarding the actions and the observations, such
    that
\vspace*{-1.6ex}
        $$(x_1,u_1,y_1,x_2, u_2, y_2, \dots) \mapsto (x_1, x_2, \dots).%
\vspace*{-0.6ex}$$
\end{example}

The two example condensers so far both have codomains composed of infinite
traces, and both work by simply discarding information from the complete trace.
This is a pattern that will recur, but is not a restriction. \gobble{Indeed, interesting
condensers can be defined so that the codomain is not even a sequence, as in
the following example.

\begin{example}[Reward]
\newcommand{\crew}{c_{\rm rew(\gamma)}}
For a robot transition system $R = (X, U, f, h, Y, X_0)$ let $r: X \times U \times X \rightarrow \Re$ be a reward function which assigns a real valued reward to the triples of $(x_i,u_i,x_{i+1})$. For $R$ and 
a constant $\gamma \in (0, 1)$ define a condenser
    $\crew: \Omega_R \to \Re$ as follows:
\vspace*{-1.6ex}
        $$ (x_1,u_1,y_1,x_2, u_2, y_2, \dots) \mapsto \sum_{i=1}^\infty \gamma^{i-1} r(x_i,u_i,x_{i+1}).%
\vspace*{-0.6ex}$$
This condenser is akin to the notion of discounted return in infinite horizon \gobblexor{reinforcement learning}{RL} or optimal control contexts 
\cite{Ber19}.
\end{example}}

\vspace*{-0.4ex}
\subsection{Information spaces}
\vspace*{-0.4ex}

Beyond $c_\H$ and $c_\X$, other condensers can be defined with codomain $\ifs^\Nat$, in which $\ifs = \powSet{X}$ is called the \emph{information space (I-space)} and its elements are referred to as \emph{information states (I-states)}. A trace of I-states resulting from such a condenser is then called \emph{an I-space trace}, usually denoted $(\is_1, \is_2, \dots)$.
These condensers capture the possibility of uncertainty in the robot's state across time.

\gobblexor{
For such condensers, we can define additional properties depending on their relationship with the true state at a given stage. The next two definitions formalize these relationships. 

\begin{definition}[Soundness of $c$]
    An I-space condenser $c : \Omega_R \rightarrow \ifs^\Nat$ is called \emph{sound} if, for any complete trace $\omega =(x_1,u_1,y_1,\dots) \in \Omega_R$ and any time $i \in \Nat$, the resulting I-state trace $(\is_1, \is_2, \dots)$ has $x_i \in \is_i$.
\end{definition}

Informally, if a condenser $c$ sound, it means that for any complete trace $\omega \in \Omega_R$, no I-state in the corresponding I-state--trace $c(\omega)$ is `missing' any true states: each I-state along the trace contains the respective state of the robot at the appropriate time slot. However, soundness does not exclude additional states that may not be consistent with the 
actions taken and the observations received by the robot. The next definition accounts for this aspect.

\begin{definition}[Tightness of $c$]
    An I-state condenser $c : \Omega_R \rightarrow \ifs^\Nat$ is called \emph{tight} if for all $\omega =(x_1,u_1,y_1,\dots) \in \Omega_R$, and any time $i \in \Nat$, the resulting I-state trace $(\is_1, \is_2, \dots)$ has the following property:
    For any state $z_i \in \is_i$, there exists a complete trace $\omega'= (x'_1,u'_1,y'_1,\dots) \in \Omega_R$ for which $z_i = x'_i$ and $c_\H(\omega) = c_\H(\omega')$.
\end{definition}

It is possible for a condenser to be any combination of sound and tight. However, in this paper, we are particularly interested in the condenser that is both sound and tight. This condenser generalizes the notion of nondeterministic I-spaces introduced by LaValle~\cite{Lav06} to infinite sequences:
}
{
We are particularly interested in the following condenser,
which generalizes the notion of nondeterministic I-spaces introduced by LaValle~\cite{Lav06} to infinite sequences:
}

\begin{example}[Nondeterministic I-state condenser]
Let $F(x_i,u_i,y_i)$ denote the set of possible states at stage $i+1$ conditioned on the state, action, and observation \gobblexor{at stage $i$, so that}{at stage $i$:}
\vspace*{-1.4ex}
\begin{equation}
\label{eq:F_transition}
    F(x_i, u_i, y_i) = \{ x' \in f(x_i, u_i) \mid  y_i \in h(x_i,u_i,x') \}.
\vspace*{-0.8ex}
\end{equation}
The \emph{nondeterministic I-state condenser} $c_{\ndet} : \Omega_R \rightarrow \ifs^\Nat$ maps a complete trace to an I-state--trace, that is,
\vspace*{-1.4ex}
$$(x_1,u_1,y_1,x_2, u_2, y_2, \dots) \mapsto (\is_1, \is_2, \dots)%
\vspace*{-0.8ex}$$
such that, starting from $\is_0 = X_0$, each subsequent $\is_k$ is \gobblexor{defined by}{just}
\vspace*{-1.2ex}
\begin{equation}
\label{eq:is_ndet}
    \is_{k} = \bigcup\limits_{x \in \is_{k-1}}\!\!\!F(x, u_k, y_k).
\vspace*{-0.8ex}
\end{equation}
\gobble{
By its construction, $c_{\ndet}$ is both sound and tight. Furthermore, it is the unique condenser satisfying both conditions as removing any element from $\is_k$ violates soundness and adding any other element to $\is_k$ violates tightness.  }
\end{example}

\gobblexor{
\smallskip
By its construction, $c_{\ndet}$ is \emph{sound} in that for any complete trace $\omega \in \Omega_R$, no I-state in the corresponding I-state--trace $c(\omega)$ is `missing' any true states.
Further, $c_{\ndet}$ is \emph{tight} in that the I-states it produces include no extraneous states, i.e., those that may not be consistent with the actions taken and the observations received by the robot. 
It is, in fact, the unique condenser satisfying both these conditions.
}

\vspace*{-0.4ex}
\section{Tasks}\label{sec:task}
\vspace*{-0.4ex}

\begin{figure}
\centering
\begin{minipage}{0.35\linewidth}
    \includegraphics[trim={0cm 0 0cm 0},clip,scale=0.75]{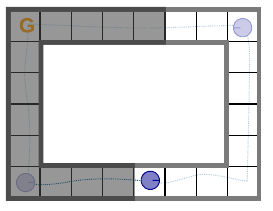}
\end{minipage}
\hfill
\begin{minipage}{0.55\linewidth}
\caption{Partly illuminated environment. The initial state is either bottom left facing east, or
top right facing west (light blue circles). The goal is indicated by \goal.}
    \label{fig:light_dark_env}
\end{minipage}
\vspace{-1.6em}
\end{figure}

\gobblexor{In the context of the formal setting introduced above, a task can be viewed as
simply a collection of complete traces, as defined next.}{Formally in our setting, a task can be viewed as
simply a collection of complete traces, as defined next.}

\begin{definition}[Task]
    A task $T$ for a robot transition system $R$ is a subset of $\Omega_R$, that is, $T \subseteq \Omega_R$.
\end{definition}

A central objective of this paper is to understand the limits of posing and
specifying tasks in ways that do not rely on the full detail of the
complete traces.  The condensers introduced in Section~\ref{sec:condenser}
provide a clean way to accomplish this.  For a given condenser $c$, the
question is whether a task of interest can be properly expressed in the
codomain of $c$, or whether the information lost when applying $c$ obscures
some aspect relevant to the task.  The next definition makes this distinction
precise.

\begin{definition}[Well-posed under $c$]\label{def:wellposed}
    A task $T \subseteq \Omega_R$ is \emph{well-posed under condenser $c$} if 
    $c^{-1}(c(T)) = T$.  
    In particular, a task $T \subseteq \Omega_R$ is called 
    \begin{enumerate}
        \item \emph{state-trace posable} if $T$ is well-posed under $c_\X$,
        \item \emph{action/observation-trace posable} if $T$ is well-posed under $c_\H$, and
        \item \emph{I-state--trace posable} if $T$ is well-posed under $c_\ndet$.
    \end{enumerate}
\end{definition}

The intuition behind the notion of \emph{well-posed} is, informally, that the codomain of $c$ is `rich enough' to describe the task properly, even in the absence of the details about each trace that are `lost' when applying the condenser $c$. 

In the following, to illustrate the notion of well-posedness under different condensers, we present examples building upon a common robot transition system.
Refer to the environment in
Figure~\ref{fig:light_dark_env} composed of an illuminated and dark cells that
the robot can be in. The state space $X=E^\Delta \times \mathcal{O}$, in which
$E^\Delta$ is the collection of grid cells and $\mathcal{O}$ are
orientations corresponding to four cardinal directions. The set $U$ of actions, which we assume to be deterministic,
include rotating $90^\circ$ in place and moving one step in the direction that
the robot is facing. The
start states $X_0$ and the goal region $X_G$ are indicated in
Figure~\ref{fig:light_dark_env}, the latter corresponding to location \goal 
with any orientation. The robot is equipped with a light sensor with the set of
observations $Y=\{\Light, \Dark, \Indet\}$ such that $h$ reports
either \Light (light) or \Indet (indeterminate) if the robot is in
a well-lit cell and it reports \Dark (dark) or \Indet
indeterminate otherwise.

\begin{example}[State-trace poseable task] \label{eg:state_trace_task}
The task is to always keep re-visiting the goal region, resulting in the set
\vspace*{-1.6ex}
\begin{multline}\label{eq:state_poseable_task}
T=\{ \omega \in \Omega_R \mid \text{elements of }  X_G \text{ appear infinitely}\\[-2pt] \qquad\text{many times in } \omega \}.
\end{multline}\vspace*{-3.6ex}

\noindent For each $z=c_\X(w)$ satisfying $\omega \in T$, $c^{-1}(z) \subseteq T$ since any $\omega' \in \Omega_R$ with $c_\X(w')=z$ will be in $T$ by construction. Furthermore, for each $\omega \in T$ there is at least one $z \in c_\X(T)$ such that $\omega \in c_\X^{-1}(z)$ since $c_\X$ is a function.
Thus, $c^{-1}_\X(c_\X(T))=T$ and $T$ is state-trace poseable.
\end{example}

\begin{example}[Action/observation-trace poseable task]\label{eg:ao_poseable}
The task is for the robot to resolve the ambiguity caused by its initialization and the geometry to determine which side of the environment it is on. Observations \Light or \Dark, if received, suffice to break the symmetry forever; therefore
\vspace*{-1.4ex}
\begin{equation}\label{eq:ao_poseable_task}
T = \{ \omega \in \Omega_R \mid \text{\Light or \Dark  appears somewhere in $\omega$ }\}. 
\end{equation}
\vspace*{-3.6ex}

\noindent Following the same reasoning as in the previous example, one can show that $T$ is well-posed under $c_\H$. 
\end{example}

\begin{example}[I-state--trace poseable task]\label{eg:I-state-not-state-poseable} Consider the task $T$ defined in Equation~\eqref{eq:ao_poseable_task}. For any $\omega=(x_1, u_1, y_1, \dots x_n, u_n, y_n, \dots)$ with $y_n$ being the first occurrence of an observation different than \Indet, the respective I-state--trace $c_\ndet(\omega)=(\is_1, \dots, \is_n, \dots)$ satisfies for all $i \geq n$ that $\is_i$ is a singleton (the dark or the light cell consistent with the actions). 
Then, consider an I-state--trace $s \in c_\ndet(w)$ with $\omega \in T$ and suppose there exists $\omega' \in c_\ndet^{-1}(s)$ with $\omega' \not\in T$. 
This cannot be possible since the I-state corresponding to the observation \Light or \Dark is a singleton indicating that an observation different than \Indet is received and that it uniquely identifies the observation (as \Light/\Dark is impossible from the symmetric state). 
Then, it is true for each $s \in c_\ndet(T)$ that $c_\ndet^{-1}(s) \subseteq T$ and that $T$ is I-state--trace poseable.
The interpretation of this task over I-state--traces is to have, at some point, the robot experience the ambiguity about its possible state vanish.
\end{example}

\begin{example}[I-state--trace poseable task]\label{eg:Istate_task}
Suppose now that, instead of the light detector, the robot has a goal detector such that $h(x,u,x')=1$ if $x \in X_G$ and $0$ otherwise. 
Define the task $T$ as in Equation~\eqref{eq:state_poseable_task} but with $\Omega_R$ incorporating this new sensor. 
This has not affected the basis of the argument made in Example~\ref{eg:state_trace_task}, so it remains state-trace poseable. 
Task $T$ is also action/observation-trace poseable since $y_i=1$ if and only if $x_i \in X_G$.    
Since $T \subseteq \Omega_R$ it respects $f$ and $h$. Then, for any I-state--trace $(\is_1, \is_2,\dots)=c_\ndet(\omega)$ with $\omega \in T$, it is always true that $\is_i \subseteq X_G$ if $y_i=1$ and $x_{i} \in X_G$, owing to Equation~\eqref{eq:F_transition}. 
The preimage $c^{-1}_\ndet((\is_1, \is_2, \dots)) \subseteq T$ since having an $\omega' \in c^{-1}_\ndet((\is_1, \is_2, \dots))$ with $\omega' \not\in T$ implies either that $\omega' \not\in \Omega_R$ or that $c_\ndet$ does not respect $f$ and $h$, neither of which is true.
Therefore, $c^{-1}_\ndet(c_\ndet(T))=T$, and $T$ is I-state--trace poseable.
\end{example}

Indeed, Example~\ref{eg:Istate_task} hints at a general rule about
I-state--trace posablity of a task that is state- and
action/observation-trace poseable, which will be stated in
Lemma~\ref{lem:uv-and-x-couple-i}. We will further analyze the relationships
between different condensers in terms of task posablity in
Section~\ref{sec:subsec_relations}.

\vspace*{-0.4ex}
\section{Specifications}\label{sec:spec}
\vspace*{-0.4ex}

\begin{definition}[Specification]
    A \emph{task specification} $\sigma$ for robot $R$ is some finite length description of a set $\Gamma_\sigma \subseteq \UOmega_R$ for which set membership can be readily computed.
    The task specified by $\sigma$ is $\Gamma_\sigma \cap\, \Omega_R$.
\end{definition}

The reason behind the use of $\UOmega_R$ is that the constraints imposed by
the robot system $R$ need not be expressed by the specification. Those are, in
some sense, obtained for free by running the robot. A specification's job is to
delineate the boundary within the set $\Omega_R$ and is free to be arbitrary
outside that set. Forcing a specification to capture the boundary of 
$\Omega_R$ within $\UOmega_R$ can be onerous and is, ultimately, unnecessary.

\vspace*{-0.4ex}
\subsection{Specification technologies} 
\vspace*{-0.6ex}

Linear Temporal Logic (\LTL) extends classical propositional logic through the
addition of operators \gobblexor{dealing with }{for }time. A formula in \LTL is a finite sequence of symbols expressed over a collection
of Boolean propositions, $\prop$, which we take to be the set of atomic
propositions.  Then, the syntax of formulae in \LTL is given recursively via
the grammar
\vspace*{-0.8ex}
$$\phi \bnfeqq \ltrue \bnfvert \lstart \bnfvert \prp{q} \bnfvert \lnot \phi \bnfvert \psi \lor \varphi \bnfvert  \lnext \phi \bnfvert  \psi \luntil \varphi,%
\vspace*{-0.6ex}$$

\noindent where the $\phi$, $\psi$, and $\varphi$ are \LTL formulae, 
$\ltrue$ is the proposition representing Boolean value `true',
$\lstart$ is the temporal operator only satisfied at the initial time,
atomic $\prp{q} \in \prop$, $\lnot$~denotes negation, $\lor$ disjunction, $\lnext$ the `next' operator, and $\luntil$ the `up until' operator.
This allows derivation 
of $\lfalse \equiv \lnot \ltrue$,
conjunction $(\phi \land \varphi) \equiv
\lnot(\lnot \phi \lor \lnot \varphi)$ via De~Morgan's law, 
implication 
$(\phi \limplies \psi) \equiv (\lnot \phi \lor \psi)$,
equivalence
$(\phi~\liff~\psi) \equiv \big((\phi~\limplies~\psi)\land(\psi~\limplies~\phi)\big)$,
and three additional temporal operators
`eventually' $\leventu \phi \equiv (\ltrue \luntil \phi)$, 
`always' $\lalways \phi \equiv (\lnot \leventu \lnot \phi)$, and
`unless' $\psi \lunless \varphi \equiv (\psi \luntil \varphi \lor \lalways \psi)$.

Semantics of \LTL formulae is provided, like modal logics, via Kripke structures. 
The standard semantics is defined over infinite sequences: $s = (P_1, P_2, P_3, \dots)$
where $P_i \subseteq \prop$ are the propositions which hold (i.e., equal $\ltrue$) at
time $i$. A sequence $s$ satisfies $\lnext \phi$ at time $i$ when $\phi$ 
evaluates to $\ltrue$ at time $i+1$. Sequence $s$ satisfies $\psi \luntil \varphi$
at time $i$ if there is some moment, either at $i$ or later, for which
$\varphi$ holds, and $\psi$ is $\ltrue$ from $i$ until the step just before then.
We write $\langle s, i \rangle \models \psi$ if and only if the sequence $s$
satisfies $\psi$ at time $i$.

In all that follows, we will consider only robot transition systems where $X$,
$U$, and $Y$ are finite.  \gobblexor{Because we associate an encoding proposition to each
state, and action, and observation, this means $\prop$ will be finite as well.}{This will mean $\prop$ is finite as well.}
Mostly, we use a change in typeface to transfer from $X\sqcup U \sqcup Y$ to their
associated propositions; 
\gobblexor{when an explicit association adds clarity, it'll be
realized via bijection $x \stackrel{\alpha}\mapsto \alpha(x) \coloneqq \prp{x} \in \prop$
(and similarly for actions or observations).}{we'll write $x \stackrel{\alpha}\mapsto \alpha(x) \coloneqq \prp{x} \in \prop$
when an explicit association adds clarity (and similarly for actions or observations).}

\begin{definition}
\label{def:state-trace-ltl}
Given state set $X=\{x', x'', \dots, x^{(m)}\}$, 
a \emph{state-trace \LTL formula} is a formula with $\prop = \{\prp{x'},\prp{x''},\dots, \prp{x^{(m)}}\}$,
with each proposition $\prp{x^{(i)}}=\alpha(x^{(i)})$, 
and with semantics defined in the standard way over infinite sequences of
sets of propositions,  but grounded to sequences 
in $\UOmega_R$ via \gobble{the following munificent condenser,}
$c_{X\prop}: \UOmega_R \rightarrow \powSet{\prop}^\Nat$:
\vspace*{-1.6ex}
$$ (x_1,u_1,y_1,x_2, u_2, y_2, \dots) \mapsto (\{\prp{x_1}\}, \{\prp{x_2}\}, \dots).
\vspace*{-0.6ex}$$
\end{definition}
\vspace*{-0.6ex}

\medskip\noindent
Intuitively, we must bridge the two treatments to allow \LTL with the usual
semantics to represent what is usually modeled.  The idea is that $c_{X\prop}$
connects the space of traces (of Definition~\ref{def:traces}) with
propositional sequences by asserting that the proposition representing state at
a given time holds when that is the state.

\begin{lemma}
\label{lem:state-trace-ltl-is-X}
Every state-trace \LTL formula expresses a state-trace posable task.
\end{lemma}
\vspace*{-0.6ex}
\begin{proof}
The state-trace \LTL formula $\phi$
describes the set $\Gamma_\phi  = \{\omega \in \UOmega_R\, |\, \langle c_{X\prop}(\omega) ,\initt \rangle \models \phi \}$,
and expresses the task $T_\phi = \Gamma_\phi \cap\, \Omega_R$.
Any $s$ satisfying $\phi$, 
i.e., where $\langle s, \initt\rangle \models \phi$,
that also comes from the image of $c_{X\prop}$,
is a sequence of singleton proposition sets; there is an injection from any such sequence into
$\X$. 
The collection of sequences in $\X$ obtained via this injection, restricted to $\omega$s originating within $\Omega_R$, is just $c_\X(T_\phi)$.
If $v \in c_\X^{-1} \circ c_\X(T_\phi)$ but $v \not\in T_\phi$ then there
must be some $w \in T_\phi$ where $c_\X(v) = c_\X(w)$.
Since $v$ and $w$ agree on the state elements of the sequence, they can
only differ on either actions or observations, or both.
But any such a difference cannot cause 
$\langle c_{X\prop}(w) ,\initt \rangle \models \phi$
while $\langle c_{X\prop}(v) ,\initt \rangle \not\models \phi$, 
so $v \in \Gamma_\phi$.
Since the domain of $c_\X$ is $\Omega_R$, necessarily $v \in \Omega_R$ and, therefore,
$v \in T_\phi$.
This establishes that $c_\X^{-1} \circ c_\X(T_\phi) \subseteq
T_\phi$. 
\end{proof}

\begin{definition}
\label{def:a-o-trace-ltl}
Given the action and observation sets $U=\{u', u'', \dots, u^{(m)}\}$ and 
$Y=\{y', y'', \dots, y^{(n)}\}$,
an \emph{action/observation-trace \LTL formula} is a formula with 
$\prop = \{\prp{u'},\prp{u''},\dots, \prp{u^{(m)}}\} \sqcup
\{\prp{y'},\prp{y''},\dots, \prp{y^{(n)}}\}$,
with 
propositions $\prp{u^{(i)}} = \alpha(u^{(i)})$ and 
$\prp{y^{(j)}} = \alpha(y^{(j)})$,
and with semantics defined in the standard way over infinite sequences of
sets of propositions,  but grounded to sequences 
in $\UOmega_R$ via \gobble{the following munificent condenser,}
$c_{H\prop}: \UOmega_R \rightarrow \powSet{\prop}^\Nat$:
\vspace*{-1.6ex}
$$(x_1,u_1,y_1,x_2, u_2, y_2, \dots) \mapsto (\{\prp{u_1},\prp{y_1}\}, \{\prp{u_2},\prp{y_2}\}, \dots).%
\vspace*{-0.6ex}$$
\end{definition}

\medskip\noindent
Unlike Definition~\ref{def:state-trace-ltl}, action/observation-trace formulae
describe properties on sequences which most robots have direct access to under
realistic assumptions. Such robots, for instance, must deal with ambiguity if
sensors fail to totally resolve state; typically, robots must pick actions
which exhibit a degree of robustness to uncertainty. The formulae of
Definition~\ref{def:a-o-trace-ltl} \gobblexor{specify associated }{stipulate }requirements directly on
action/observation sequences; the area of sensor-based planning, not to
mention reactive controllers, often perform activities to achieve tasks \gobblexor{which may be specified in this form}{matching this form}.

\vspace*{-0.3ex}
\begin{lemma}
\label{lem:history-trace-ltl-is-H}
Every action/observation-trace \LTL formula expresses an action/observation-trace posable task.
\end{lemma}
\vspace*{-0.4ex}
\begin{proofsketch}
Proceed analogously to Lemma~\ref{lem:state-trace-ltl-is-X}.
\end{proofsketch}
\vspace*{-0.4ex}

\medskip\noindent
Robots processing streams of actions and observations may treat uncertainty in
a variety of ways.  Some cases, albeit relatively few realistic ones, are
`certainty equivalent'\,\cite[p.~76]{Ber19}, so that some surrogate state,
which although potentially not the actual state known by the system, may
suffice to be used by the robot to reason about how to act. More generally, one
may integrate data into an estimate describing what is known (or `believed')
from the history of actions and observations. In some instances this estimate,
or even reduced descriptors derived therefrom, may be sufficient for the robot
to act.  \gobblexor{Perhaps unsurprisingly,}{Unsurprisingly,} one may wish to specify tasks for the robot in
terms of these estimates and this motivates the \gobblexor{definition that follows.}{following.}

\begin{definition}
\label{def:desc-Istate-trace-ltl}
Given state set $X=\{x', x'', \dots, x^{(m)}\}$, 
and an information space $\ifs=\{\is', \is'', \dots, \is^{(m)}\}$, 
an \emph{I-state--trace \LTL formula} is a formula with 
$\prop = \{\prp{x'},\prp{x''},\dots, \prp{x^{(m)}}\}$,
with each proposition $\prp{x^{(i)}} = \alpha(x^{(i)})$ being a plausible true state consistent with the data received so far, and with
semantics defined in the standard way over infinite sequences of sets of
propositions,  but grounded to sequences 
in $\UOmega_R$ via \gobble{the following munificent condenser,}
$c_{I\prop}: \UOmega_R \rightarrow \powSet{\prop}^\Nat$:
\vspace*{-1.6ex}
$$(x_1,u_1,y_1,x_2, u_2, y_2, \dots) \mapsto (I_1, I_2, \dots),%
\vspace*{-0.6ex}$$
where $I_n \coloneqq \{ \prp{x'} \mid x' \in \is_n\}$, 
 with $\is_n$ from Equation~\eqref{eq:is_ndet}.
\end{definition}

\gobble{A few remarks are in order.} The reader may have expected a definition
directly analogous
to the previous two, i.e., over symbols $\prp{\iS_n} = \alpha(\is_n)$.
Such a definition is indeed
possible,\gobble{\footnotemark{}} but the present definition has two \gobble{specific }merits:
First, it is more economical in that \LTL's semantics already permits sets of
propositions to hold and, while the previous two definitions have not had 
cause to make use of this facility, the I-states \gobble{naturally do so most
effectively.}{do so naturally.}  Secondly, it emphasizes the vital conceptual point that two
formulae written in terms of propositions corresponding to states (like
$\prp{x'}$) will be indistinguishable at their surface level: whether the
formulae represent statements about the ground-truth state or the robot's
belief of a feasible state depends entirely on how the 
propositions are grounded to~$\UOmega_R$.

\gobble{
\footnotetext{One simply forms, mutatis mutandis, the
munificent condenser,
$(x_1,u_1,y_1,x_2, u_2, y_2, \dots) \mapsto (\{\prp{\iS_1}\}, \{\prp{\iS_2}\}, \dots).$
Formulae may be expressed over such propositions. Any formula
in terms of $\prp{\iS}$ propositions can be turned into an equivalent one
of the form in Definition~\ref{def:desc-Istate-trace-ltl} and vice versa.
The equivalence is as follows:
any $\prp{\iS_n} \equiv \big(\bigwedge_{x_\ell \in \is_n} \prp{x_\ell}\big)  \wedge \big(\bigwedge_{x_k \in X\setminus\is_n}\!\!\!\lnot\prp{x_k}\big)$,
and similarly any $\prp{x_n} \equiv \big(\bigwedge_{\is_j \in \{\is \in \ifs \,|\, x_n \in \is\}} \prp{\iS_j}\big)
\wedge \big(\bigwedge_{\is_m \in \{\is \in \ifs \,|\, x_n \not\in \is\}} \!\lnot\prp{\iS_m}\big)$. \label{foot:transformation}
}
}

\begin{lemma}
\label{lem:istate-trace-ltl-is-I}
Every I-state--trace \LTL formula expresses an \mbox{I-state}--trace posable task.
\end{lemma}
\vspace*{-0.5ex}
\begin{proofsketch}
Again, similar to that of Lemma~\ref{lem:state-trace-ltl-is-X}.
\end{proofsketch}

\vspace*{-0.0ex}
\begin{corollary}
\label{cor:no-formulae}
There are no state-, action/observation-, or I-state--trace \LTL formulae for 
any tasks which are not state-, action/observation-, or I-state--trace posable, respectively.
\end{corollary}
\vspace*{-0.5ex}
\begin{proofsketch}
Contrapositives of Lemmata~\ref{lem:state-trace-ltl-is-X}, \ref{lem:history-trace-ltl-is-H}, \ref{lem:istate-trace-ltl-is-I}.
\end{proofsketch}

\vspace{-0.3em}
\subsection{Relationships}\label{sec:subsec_relations}
\vspace{-0.3em}

\begin{definition}[Equivalent \LTL formulae]
Any two \LTL formulae $\psi$ and $\phi$ are \emph{equivalent for robot
system $R$} if and only if 
they specify the same task, i.e., 
$\Gamma_\psi \cap\, \Omega_R = \Gamma_\phi \cap\, \Omega_R$.
\end{definition}

\begin{example}\label{eg:Table_1_1_1} 
Refer back to Example~\ref{eg:Istate_task} which presented a task that was
state-trace, action/observation-trace, and I-state--trace poseable. This task
is also expressible in the respective domains with respective \LTL formulae.
\end{example}

\begin{lemma}
\label{lem:i-state-posable-is-UY-posable}
Any I-state--trace posable task is action/\hspace{0pt}observation-trace posable.
\end{lemma}
\begin{proof}
Given robot system $R$,
for all $v, v' \in \Omega_R$, we will have that $c_\H(v) = c_\H(v') \implies c_\ndet(v) = c_\ndet(v')$ as
the definition of $c_\ndet$, via Equations~\eqref{eq:F_transition} and \eqref{eq:is_ndet}, uses $X_0$ and then only
the sequence of $U$ and $Y$ elements, i.e., it processes any two traces equal under $c_\H$ identically. 
\end{proof}

\medskip
In its essence, the preceding shows that tasks posed in terms of I-states are no finer than
those posable in terms of actions/observations. The phrase `no finer' can be tightened
to `coarser' as the next example illustrates this can be strict.

\begin{figure}[t]
\centering
\begin{minipage}{0.5\linewidth}
\vspace{-3.2em}
    \includegraphics[trim={0cm 0 0cm 0},clip,scale=0.78]{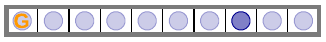}
\end{minipage}
\vspace{-1.2em}
\begin{minipage}{0.44\linewidth}
\caption{A left/right moving robot, whose initial position is entirely unknown, occupies one of  $10$ cells in a \gobble{horizontal }corridor.
    \label{fig:corridor-env}}
\end{minipage}
\vspace{-1.0em}
\end{figure}

\begin{example}\label{eg:Table_0_1_0} 
Consider the corridor environment in Figure~\ref{fig:corridor-env}. The state
space is $X=\{1,2,\dots,10\}$, and the robot can start anywhere, so $X_0 = X$. 
The robot is equipped with two actions: move \underline{R}ight or \underline{L}eft, which
give $x \stackrel{\text{R}}{\mapsto} \max(x+1,10)$ and $x \stackrel{\text{L}}{\mapsto} \min(1,x-1)$, respectively. 
Suppose that the robot has a sensor which returns the distance to the left-hand wall, but
with a noisy sign, modeled via $h(x, u, x') = \{x', -x'\}$ and $Y = \{-10,\dots, 10\}
\setminus\{0\}$.  Task $T$ consists of those traces in which some non-positive
sensor reading appears (i.e., the robot obtains concrete evidence of
noise-induced mis-measurement). It is straightforward to specify in action/observation-trace \LTL.
But states (or even sets of states) provide no ability to discriminate between receipt of $\pm n \in Y$, so
the problem is neither state- nor I-state--trace poseable.
\end{example}

The following two examples will show that this `ordering' intuition doesn't extend
analogously to state-trace poseable (and \LTL specified) tasks. 

\begin{example}\label{eg:Table_0_1_1}
Refer again to the task $T$ in Example~\ref{eg:ao_poseable} that was shown to
be action/observation and I-state--trace poseable. It is also expressible with
respective \LTL formulae. However, $T$ is not state-trace poseable. To show
this, consider $\omega \in T$ for which \Light appears only once and the
respective state-trace $z = c_\X(\omega)$. The preimage $c^{-1}_\X(z)$ contains
also the trace $\omega'$ with \Indet instead, since that observation 
is certainly possible there as well. Therefore, $c^{-1}_\X(z) \not\subseteq T$ and
$T$ is not state-trace poseable. 
No state-trace \LTL formula exists for this task, by Corollary~\ref{cor:no-formulae}. 
\end{example}

\begin{example}\label{eg:Table_1_0_0} 
Revisit the environment and robot in Example~\ref{eg:Table_0_1_0}, but now remove the robot's sensor.
Here, consider a task $T$ that is achieved if and only if the robot visits the \goal location at the far left of Figure~\ref{fig:corridor-env}. 
Then $T \subseteq \Omega_R$ is simply the collection of traces $\omega$ in which `$1$' appears somewhere in $c_\X(\omega)$.
In state-trace \LTL, it is specified as $\leventu \prp{x_1}$.
This task is neither action/observation-trace nor I-state--trace poseable: any
sequences where, in any prefix, the total number of `L' actions does not exceed
the total number of `R' actions so far appearing, will confuse initial
locations $1$ and $2$. The former must be in $T$, the latter ought not to be.
\end{example}

\begin{lemma}
Let $\psi_x$ be a state-trace \LTL formula specifying task $T$ for robot transition system~$R$.
Then there is an action/observation-trace \LTL formula $\psi_{u/y}$ expressing task $T$ if 
it is action/observation-trace posable.
\label{lem:state-spec-to-action-obs-spec}
\end{lemma}
\begin{proofsketch}
From $\autom{A}$, the B\"uchi automaton describing $\psi_x$, and system $R$, \gobblexor{use Construction~\ref{algo:buchi-rts-product} to form a new nondeterministic B\"uchi automaton $\autom{A'}$.
And, we then establish that $T = \Gamma_{\autom{A}} \cap\, \Omega_R = \Gamma_{\autom{A'}}\cap\, \Omega_R$.
}{take a product to form nondeterministic B\"uchi automaton $\autom{A'}$, and establish $T = \Gamma_{\autom{A}} \cap\, \Omega_R = \Gamma_{\autom{A'}}\cap\, \Omega_R$.}
Next, it remains to show that $\autom{A'}$ corresponds to some formula $\psi_{u/y}$,
as the class of B\"uchi automata is strictly larger than the class realized by
\LTL formulae. Diekert and Gastin provide
\emph{counter-freeness}\,\cite[p.~286]{diekert08first} as a characterization
suitable of nondeterministic B\"uchi automata realizable via formulae.
We establish that $\autom{A'}$ is counter-free by supposing the contrary,
and showing this would cause $\autom{A}$ to not be counter-free.
\end{proofsketch}

\begin{remark}
To collect the previous results:
if the task $T$ is action/observation-trace posable and one has a 
state-trace \LTL formula specifying $T$ (hence, via Lemma~\ref{lem:state-trace-ltl-is-X}, it is state-trace posable)
then Lemma~\ref{lem:state-spec-to-action-obs-spec} says there is an action/observation-trace \LTL formula for $T$.
Conversely, should $T$ not be action/observation-trace posable,
Corollary~\ref{cor:no-formulae} indicates that no action/observation-trace
\LTL formula will specify it.
\gobblexor{Since specifiability is narrower than posablity, it is worth observing that the result allows
non-specifiability to give conclusions about non-posablity:}{As specifiability is narrower than posablity, it is interesting that 
non-specifiability allows conclusions about non-posablity:---}
with state-trace formula $\psi_x$ specifying task $\hat{T}$, and knowledge that
no $\psi_{u/y}$ specifies $\hat{T}$, then one can conclude
task $\hat{T}$ is not action/observation-trace posable.
\end{remark}

\begin{corollary}
\label{cor:all-I-states-formula-have-UY}
Let $\psi_i$ be an I-state--trace \LTL formula specifying task $T$ for\gobble{ robot 
transition} system~$R$.  Then, there is an action/\allowbreak observation-trace \LTL formula
$\psi_{u/y}$ expressing~task~$T$.
\end{corollary}
\begin{proofsketch}
From $R = (X, U, f, h, Y, X_0)$ construct 
system $R' = (X', U, F', H', Y, \{X_0\})$,
where $X' = \powSet{X}\setminus\{\emptyset\}$, and, 
using Equation~\eqref{eq:F_transition}, $F'(A, u, y) = \bigcup_{x \in A}\!F(x, u, y)$ along with suitable $H$.
Then, I-state--traces on $R$ are state-traces on $R'$, and
$T$ is action/observation-trace posable in $R'$ if and only if it is in~$R$. 
Next, transform $\psi_i$, written in terms of
$\prp{x_i}$s, into $\psi_i'$ over $\prp{\iS_j}$s
through the (purely syntactic) rewriting\gobble{ described in
footnote~\ref{foot:transformation}}.
Lemma~\ref{lem:state-spec-to-action-obs-spec} implies some 
$\psi_{u/y}'$ exists provided the task is action/observation-trace
posable for $R$ or $R'$.
Lemma~\ref{lem:i-state-posable-is-UY-posable} indicates it must be\gobble{ for $R$}.
\end{proofsketch}

\begin{lemma}
\label{lem:uv-and-x-couple-i}
Every task that is state-trace and action/\allowbreak observation-trace posable is also I-state--trace posable.
\end{lemma}
\begin{proofsketch}
Suppose the contrary, then
there must be a pair of
sequences, $w \in T$ and $\bar{w}\not\in T$ that look identical under $c_\ndet$, but can be distinguished under 
$c_\H$ and $c_\X$. Consider a sequence $w'$, with the states from $w$ and the action/observations of $\bar{w}$.
Sequence $w'$ is in $\Omega_R$ because it generates the same I-state--trace as $w$ and $\bar{w}$.
But under $c_\H$, $w'$ is in $T$, while under $c_\X$ $w'$ is out.
\end{proofsketch}

\section{Punchline: Limits of specifiability}\label{sec:table}
Consider the following family of claims regarding the limits of specifiability which may be constructed by resolving the expressible/inexpressible choices in the template below:

\vspace{0.05in}
\textbf{Claims 1--7:} \quad \emph{There exist a robot transition system $R$ (with finite $X$, $U$, $Y$) and some task $T$ for which $T$ is
    \underline{(in-)expressible} in state-trace LTL formulae,
    \underline{(in-)expressible} in action/observation-trace LTL formulae, and
    \underline{(in-)expressible} in I-state--trace LTL formulae.}
\vspace{0.05in}

Eagle-eyed readers will note that all seven of these claims have already been resolved in the preceding text. 
Here is a summary of results in table form.

\newcommand*\tabnote[1]{\small{#1}}

\begin{tabular}{p{30pt}p{30pt}p{20pt}cl}
\hspace*{-5pt}{\textbf{\footnotesize State}} & 
\hspace*{-10pt}{\textbf{\footnotesize Act./Obs.}} & 
\hspace*{-10pt}{\textbf{\footnotesize I-state}} & 
\hspace*{-08pt}\textbf{\footnotesize Claim holds} & \textbf{\footnotesize Details}     \\[-1pt] \hline
\textos{0} & \textos{0}  & \textos{1}   &\hspace*{-3pt} \footnotesize{\textsc{No}}  & \tabnote{Corollary~\ref{cor:all-I-states-formula-have-UY}}\\[-2pt]
\textos{0} & \textos{1}  & \textos{0}   &\hspace*{-3pt} \footnotesize{\textsc{Yes}} & \tabnote{Example~\ref{eg:Table_0_1_0}}\\[-2pt]
\textos{0} & \textos{1}  & \textos{1}   &\hspace*{-3pt} \footnotesize{\textsc{Yes}} & \tabnote{Example~\ref{eg:Table_0_1_1}}\\[-2pt]
\textos{1} & \textos{0}  & \textos{0}   &\hspace*{-3pt} \footnotesize{\textsc{Yes}} & \tabnote{Example~\ref{eg:Table_1_0_0}}\\[-2pt]
\textos{1} & \textos{0}  & \textos{1}   &\hspace*{-3pt} \footnotesize{\textsc{No}}  & \tabnote{Corollary~\ref{cor:all-I-states-formula-have-UY}}\\[-2pt]
\textos{1} & \textos{1}  & \textos{0}   &\hspace*{-3pt} \footnotesize{\textsc{No}}  & \tabnote{Lemma~\ref{lem:uv-and-x-couple-i}}\\[-2pt]
\textos{1} & \textos{1}  & \textos{1}   &\hspace*{-3pt} \footnotesize{\textsc{Yes}} & \tabnote{Example~\ref{eg:Table_1_1_1}}\\ \hline
\end{tabular}

\vspace{0.1em}
\section{Conclusion and Outlook}\label{sec:conc}
\vspace{-0.2em}

The present contribution is fairly unusual as a robotics paper in that it deals
with robot tasks without engaging in the question of how to 
identify optimal controls, select good actions,  or find plans that allow a
system to solve such tasks: the paper is concerned only with how one might
characterize behavior satisfying the requirements of some task one might have
in mind. It is declarative in its entirety.  We do not discuss how a robot
might perform actions to realize a task---indeed, the formalism being built on
traces means, perhaps surprisingly, that no notion of the present or current
time appears at all.  It is our hope that this quizzical fact might stimulate
fruitful discussion.

\bibliographystyle{IEEEtranS}
\bibliography{bibliography}

\end{document}